%% file: main.tex
\pdfoutput=1

\documentclass[11pt]{article}

\usepackage{fullpage}
\usepackage{mathtools,amsthm,amsfonts}
\usepackage{enumitem}
\usepackage{hyperref}
\usepackage{xcolor}
\usepackage{graphicx}
\usepackage[algo2e,ruled,linesnumbered]{algorithm2e}
\usepackage{float}
\usepackage{authblk}
\hypersetup{
	colorlinks	= true,
	urlcolor	= orange,
	citecolor	= blue,
	linkcolor	= red
}

\title{
    Exponential Weights Algorithms for Selective Learning\thanks{
    We would like to thank Jay Mardia for helpful discussions during this project, especially for pointing us to Bregman divergence and self-concordance that underlie Theorem~\ref{thm:mean-prediction}. This work was supported by NSF Awards CCF-1704417 and AF-1813049, DOE Award DE-SC0019205, and ONR Young Investigator Award N00014-18-1-2295.
}
}
\date{}
\author{Mingda Qiao}
\author{Gregory Valiant}
\affil{\texttt{\{mqiao,valiant\}@stanford.edu}}
\affil{Stanford University}

\newcommand{\D}{\mathcal{D}}
\newcommand{\eps}{\epsilon}
\newcommand{\Ex}[1]{\mathbb{E}\left[#1\right]}
\newcommand{\E}{\mathcal{E}}
\newcommand{\F}{\mathcal{F}}
\newcommand{\I}[1]{\mathbb{I}\left[#1\right]}
\newcommand{\Ical}{\mathcal{I}}

\newcommand{\KL}[2]{\mathrm{KL}(#1||#2)}
\renewcommand{\L}{\mathcal{L}}

\newcommand{\poly}{\operatorname*{poly}}
\newcommand{\polylog}{\operatorname*{polylog}}
\newcommand{\pr}[1]{\Pr\left[#1\right]}
\newcommand{\R}{\mathbb{R}}
\newcommand{\T}{\mathcal{T}}
\newcommand{\U}{\mathcal{U}}
\newcommand{\X}{\mathcal{X}}
\newcommand{\Y}{\mathcal{Y}}
\newcommand{\Z}{\mathcal{Z}}

\newtheorem{theorem}{Theorem}
\newtheorem{problem}{Problem}
\newtheorem{lemma}{Lemma}
\newtheorem{definition}[lemma]{Definition}
\newtheorem{example}[lemma]{Example}
\newtheorem{remark}[lemma]{Remark}
\newtheorem{proposition}[lemma]{Proposition}

\begin{document}

\maketitle

\begin{abstract}%
    We study the selective learning problem introduced by~\cite{qiao2019theory}, in which the learner observes $n$ labeled data points one at a time. At a time of its choosing, the learner selects a window length $w$ and a model $\hat\ell$ from the model class $\L$, and then labels the next $w$ data points using $\hat\ell$. The \emph{excess risk} incurred by the learner is defined as the difference between the average loss of $\hat\ell$ over those $w$ data points and the smallest possible average loss among all models in $\L$ over those $w$ data points.
    
    We give an improved algorithm, termed the \emph{hybrid exponential weights} algorithm, that achieves an expected excess risk of $O((\log\log|\L| + \log\log n)/\log n)$. This result gives a doubly exponential improvement in the dependence on $|\L|$ over the best known bound of $O(\sqrt{|\L|/\log n})$. We complement the positive result with an almost matching lower bound, which suggests the worst-case optimality of the algorithm. 

    We also study a more restrictive family of learning algorithms that are \emph{bounded-recall} in the sense that when a prediction window of length $w$ is chosen, the learner's decision only depends on the most recent $w$ data points. We analyze an exponential weights variant of the ERM algorithm in~\cite{qiao2019theory}. This new algorithm achieves an expected excess risk of $O(\sqrt{\log |\L|/\log n})$, which is shown to be nearly optimal among all bounded-recall learners. Our analysis builds on a generalized version of the selective mean prediction problem in~\cite{drucker2013high,qiao2019theory}, which may be of independent interest.
\end{abstract}

\input{intro}

\input{hybrid}

\input{bounded-recall}

\bibliographystyle{alpha}
\bibliography{main}

\appendix

\input{missing-proofs}

\end{document}

%% file: intro.tex
\section{Introduction}

We consider learning in an online setting: a sequence of labeled data points becomes available in an online fashion, and the learner is asked to choose a model from a given model class $\F$ (e.g., the set of linear models, or neural networks of a certain architecture) to label the unseen data points. If the data are identically and independently distributed, the problem is equivalent to the supervised learning setting. However, this i.i.d.\ assumption rarely holds in practice due to possible shifts in the data distribution or even malicious data corruption.

In online learning theory, this non-i.i.d.\ feature of the problem is typically modeled by allowing the data sequence to be arbitrary, and then evaluating the learning performance by comparing it to the best \emph{fixed} model in the model class $\F$ for the whole data sequence. In many of these formulations, the learner is required to predict a label for every single data point.

In this work, we study the online learning setting from a different angle, termed \emph{selective learning}, that allows the learner to choose an arbitrary window of the data sequence to make its prediction. The learner observes the data points one at a time, and at a time of its choosing, it selects a window length $w$ and a certain model $\hat f \in \F$, and then predicts that $\hat f$ is a good model for the next $w$ unseen data points. The learner is evaluated in terms of the \emph{excess risk}, i.e., the gap between the average loss of $\hat f$ on the $w$ data points and the smallest possible loss among all the models in $\F$ on those $w$ data points. Importantly, the benchmark to which we compare the learner's performance is the best model \emph{for the specific prediction window}, rather than the best one for the whole data sequence (as in the usual definition of regret in online learning).

This selective learning setting models many high-stakes decision making scenarios where one can freely decide when to deploy a certain model, but after that the switching cost is prohibitively high; once a decision is made, the learner has to commit to the chosen model throughout the prediction window. This is in the same spirit as the work on bandit problems with switching costs (see e.g.~\cite{agrawal1988asymptotically,jun2004survey,cesa2013online} and the references therein), where the agent is charged a cost whenever it changes its action. In addition, deploying the same model over the chosen prediction window can be more desirable when interpretability is important: if $\F$ is a family of decision trees, the predictions made by each single decision tree in $\F$ can be easily interpreted, whereas frequently switching between multiple trees is much less transparent.

Similarly to the usual online learning setting, no distributional assumptions are made on the data sequence---the data points and their labels can be adversarially chosen, as long as they are not adaptive to the learner's actual prediction. Although this setting might seem too general for any non-trivial learning guarantees to exist, perhaps surprisingly, the main results of \cite{drucker2013high,qiao2019theory} show that the prediction counterpart of this problem can be accurately done; we discuss their results in more detail in Section~\ref{sec:related}. 

The intuition behind the positive results in~\cite{drucker2013high,qiao2019theory} is the Ramsey theoretic observation that any sufficiently long sequence of bounded numbers must exhibit a certain level of ``structure'' at some timescale.  Specifically, there must be some timescale at which the moving averages do not vary too wildly. 
Our work can be viewed as analyzing the analogous question in higher dimension, where we ask the extent to which every high-dimensional sequence must exhibit some type of structure at some timescale, and the type of structure in question is mediated by the model family, $\F$.


\subsection{Problem Setting}
Let $\X$ and $\Y$ denote the instance space and the label space, and define $\Z = \X \times \Y$ as the space of labeled data. The model class $\F$ is a family of functions (also called models) that map $\X$ to $\Y$. A bounded loss function $\ell:\Y\times\Y\to[0,1]$ is specified, such that $\ell(\hat y, y)$ is the loss for predicting label $\hat y$ on an instance with true label $y$. Given the loss function $\ell$, it is convenient to view the model class $\F$ as a family $\L$ of functions mapping labeled data to losses: each model $f_i \in \F$ corresponds to a function $\ell_i: \Z \to [0, 1]$ defined as $\ell_i(x, y) = \ell(f_i(x), y)$. Thus, we will use $\F$ and $\L$ interchangeably to denote the model class in the following.

We formally define the selective learning problem over model class $\L$ as follows.
\begin{problem}[Selective Learning]\label{prob:learn}
    A labeled data sequence $z \in \Z^n$ of length $n$ is chosen at the beginning. The learner observes $z_1, z_2, \ldots, z_n$ one by one. At any time step $t \in \{0, 1, \ldots, n-1\}$, after seeing the first $t$ data points, the learner can specify a window length $w \in [n - t]$ and a model $\hat\ell \in \L$. The learning procedure then terminates, and the excess risk of the learner is defined as $\frac{1}{w}\sum_{i=1}^{w}\hat\ell(z_{t+i}) - \inf_{\ell\in\L}\frac{1}{w}\sum_{i=1}^{w}\ell(z_{t+i})$. The learner must make one such action (specifying a window length and a model) before all the $n$ data points are revealed.
\end{problem}

The above setting is analogous to the \emph{selective prediction} model studied by~\cite{drucker2013high,qiao2019theory}, where a predictor predicts certain pre-specified statistics of the unseen data. This selective learning setting was also informally introduced in~\cite{qiao2019theory}, termed the problem of fitting future data.

\begin{remark}[Comparison to online learning]
    While the definition of Problem~\ref{prob:learn} closely resembles that of the ``experts problem'' in online learning (see e.g.~\cite{littlestone1989weighted,cesa1997use,cesa2006prediction}), we highlight two important differences: (1) In online learning, the predictor is required to predict at every single time step. In contrast, the learner in our setting is selective in the sense that it makes only one prediction, which could span any window of its choice; (2) In the experts problem, the performance of the learner is compared to the best expert for the whole sequence in hindsight, while in our setting, the ``benchmark'' is the best model for the specific window.
    
    More recently, the notion of ``adaptive regret'' has been studied in the online learning literature~\cite{hazan2009efficient,adamskiy2016closer}. The adaptive regret of the learner is defined as the maximum regret that it incurs among all contiguous sub-intervals of the time horizon. The selective learning setting is similar to the adaptive regret formulation in that both evaluate the learner based on the best expert/action for each specific interval; however, the selective learning setting only considers the interval chosen by the learner, while the work on adaptive regret typically requires the learner to act at every step. Thus, the results in the two settings are qualitatively different and incomparable.
\end{remark}

We consider the following two families of selective learning algorithms, \emph{non-adaptive} learners and \emph{bounded-recall} learners, which choose the prediction windows and the models in a more restrictive way and thus allow simpler analyses.

\begin{definition}[Non-Adaptive and Bounded-Recall Learners]\label{def:learner}
    A selective learning algorithm is non-adaptive if it decides the prediction window before seeing the data sequence. Furthermore, an algorithm is bounded-recall if it is non-adaptive and, whenever it chooses a prediction window $t + 1, t + 2, \ldots, t + w$, it ignores all but the $w$ data points $z_{t-w+1}, z_{t-w+2}, \ldots, z_t$ immediately before the window.
\end{definition}

The new algorithms in this paper are all non-adaptive; as we show, the performance of one of these non-adaptive learners is comparable to that of the best adaptive one for a wide range of parameters. The bounded-recall property, on the other hand, turns out to be a much more stringent constraint on the learner that leads to an exponential increase in the optimal dependence on $|\L|$. We also note that all the selective prediction algorithms studied by~\cite{drucker2013high,qiao2019theory} are both non-adaptive and bounded-recall.

The bounded-recall algorithms are similar to the ``bounded-recall strategies'' in repeated games that choose the action in the current round based on a bounded number of the most recent rounds~\cite{lehrer1988repeated}. In the context of selective learning, these bounded-recall learners are memory-efficient in the sense they only need to store as many data points as the length of the prediction window.

The following prediction problem generalizes the selective prediction setting to higher dimensions and more general loss functions, and will be useful for our analysis of selective learning.

\begin{problem}[Generalized Mean Prediction]\label{prob:mean}
    Let $S \subseteq \R^d$ be a convex set and $D: S\times S \to \R$. A sequence $x \in S^n$ of length $n$ is chosen at the beginning. The predictor observes $x_1, x_2, \ldots, x_n$ one by one. At any time step $t \in \{0, 1, \ldots, n - 1\}$, after seeing the first $t$ entries of $x$, the predictor may choose a window length $w \in [n - t]$ and predict the average of $x_{t+1}, \ldots, x_{t+w}$. Once a prediction is made, the procedure terminates immediately. If the prediction is $\hat x$ and the actual average is $\bar x = \frac{1}{w}\sum_{i=1}^{w}x_{t+i}$, the predictor incurs a loss of $D(\hat x, \bar x)$.
\end{problem}

Recall that for a differentiable convex function $f$, the Bregman divergence $D_f$ is defined as $D_f(x, y) = f(x) - f(y) - \nabla f(y)^{\top}(x-y)$ and is non-negative. In the analysis of selective learning algorithms, the excess risk turns out to be closely related to the Bregman loss defined by a convex function over $\R^{|\L|}$, where each coordinate corresponds to one model in $\L$. Thus, we will focus on the case that the loss function $D$ is defined by a Bregman divergence: for differentiable and convex $f: S \to \R$, we define the $f$-loss as $D(\hat x, \bar x) \coloneqq D_f(\hat x, \bar x) + D_f(\bar x, \hat x)$.

\begin{example}\label{ex:mean-est}
    For $S = [0, 1]$ and $f(x) = \frac{1}{2}x^2$, Problem~\ref{prob:mean} with $f$-loss is exactly the mean estimation problem with squared loss studied by~\cite{drucker2013high,qiao2019theory}.
\end{example}

\subsection{Overview of Results}

\paragraph{Tight bounds for selective learning.} Our main result is the following upper bound on the optimal excess risk in selective learning.
\begin{theorem}\label{thm:hybrid}
    There exists a selective learning algorithm that, on model class $\L$ and an arbitrary sequence of $n$ data points, achieves an excess risk of $O\left(\frac{\log\log|\L| + \log\log n}{\log n}\right)$ in expectation.
\end{theorem}

We prove Theorem~\ref{thm:hybrid} in Section~\ref{sec:hybrid} using a new algorithm termed the \emph{hybrid exponential weights}. This algorithm is non-adaptive but not bounded-recall in the sense of Definition~\ref{def:learner}. The key idea behind the algorithm is to apply the classic exponential weights algorithm to a hypothetical instance of the experts problem, in which every single time step corresponds to a carefully chosen block of data points.  

Theorem~\ref{thm:hybrid} resolves the open problem raised by~\cite{qiao2019theory} regarding whether an excess risk of $\polylog(|\L|)/\polylog(n)$ can be achieved. In fact, the theorem gives a much better $\log\log|\L|$ dependence, which might be surprising considering that the learning error bounds in many learning settings (e.g., PAC learning and online learning) are polynomial in $\log |\L|$, and $\log |\L|$ can be naturally interpreted as the ``dimension'' of the model class. In Section~\ref{sec:discuss}, we highlight a few other contrasts between the results in selective learning and PAC learning.

We complement Theorem~\ref{thm:hybrid} with a lower bound result showing that an $\Omega(\log\log|\L|/\log n)$ excess risk is unavoidable when either of the following two holds: (1) the learner is non-adaptive in the sense of Definition~\ref{def:learner}, i.e., it chooses the prediction window independently of the data sequence; or (2) $|\L|$ is large compared to $n$.

\begin{theorem}\label{thm:learn-lower}
    For $n \ge 1$, $m \ge 2$ and any selective learning algorithm, suppose that either (1) the algorithm is non-adaptive or (2) $m \ge n^{\omega(\log\log n)}$.
    Then, there exists a selective learning instance with $|\L| = m$ and sequence length $n$ such that the expected excess risk of the algorithm is at least $\Omega\left(\min\left\{\frac{\log\log |\L|}{\log n}, 1\right\}\right)$.
\end{theorem}

We defer the proof of Theorem~\ref{thm:learn-lower} to  Appendix~\ref{sec:missing-lower}. The first case of the theorem is proved by constructing $\frac{\log n}{\log\log|\L|}$ hard instances such that for each instance, when the window length chosen by the learner falls into a certain interval, the expected excess risk is $\Omega(1)$. Moreover, the $\frac{\log n}{\log\log|\L|}$ intervals form a partition of $[1, n]$. Since the window length has to fall into one of these intervals with probability $\Omega\left(\frac{\log\log |\L|}{\log n}\right)$, the lower bound would follow from the non-adaptivity of the learner. The proof of the second case is more technical and builds on the construction in~\cite{qiao2019theory} of a distribution over sequences that exhibits a significant amount of anti-concentration at every timescale. We analyze the selective learning instance defined by drawing $m = |\L|$ independent sequences from the above distribution, and setting the loss of the $i$-th model on $z_1, z_2, \ldots, z_n$ to be exactly the $i$-th sequence.

\paragraph{A near-optimal bounded-recall algorithm.} \cite{qiao2019theory} proposed an empirical risk minimization (ERM) algorithm for selective learning and proved a weaker excess risk bound of $O(\sqrt{|\L|/\log n})$. The ERM algorithm picks a random prediction window appropriately, and then chooses the model with the best performance on the interval immediately before the prediction window. Unfortunately, this $\poly(|\L|)$ dependence is shown to be tight for ERM: there exists a family of instances where $|\L| = \Theta(\log n)$ and ERM incurs a constant excess risk. Roughly speaking, this is because ERM is too sensitive to small differences between the models, and thus can be tricked by the adversary into incurring a high excess risk.

Given the success of exponential weights in proving Theorem~\ref{thm:hybrid}, it is natural to ask whether a lower excess risk can be achieved using a ``soft'' version of ERM. This variant is termed the \emph{bounded-recall exponential weights} algorithm\footnote{The algorithm is indeed a bounded-recall learner in the sense of Definition~\ref{def:learner}.} and formally defined as Algorithm~\ref{alg:exp-weights}. The key change is that the decision rule of the algorithm is much smoother than that of ERM: instead of choosing the model that exactly minimizes the ``empirical'' risk, the algorithm randomly draws a model with probability exponential in the negative empirical risk. In the following, $\U(S)$ denotes the uniform distribution over a finite set $S$.

\begin{algorithm2e}[H]
    \caption{Bounded-Recall Exponential Weights} \label{alg:exp-weights}
    \KwIn{Model class $\L=\{\ell_1,\ell_2,\ldots,\ell_{|\L|}\}$ and online access to data sequence $z \in \Z^n$. Parameter $\alpha > 0$.}
    $k \gets \lfloor\log_2 n\rfloor$; Draw $k' \sim  \U([k])$\;
    Draw $t \sim \U(\{0, 2^{k'}, 2\cdot 2^{k'}, 3 \cdot 2^{k'}, \ldots, 2^k - 2^{k'}\})$\;
    Observe $z_1, z_2, \ldots, z_{t+2^{k'-1}}$\;
    \lFor{$i = 1, 2, \ldots, |\L|$} {
	    $u_i \gets \frac{1}{2^{k'-1}}\sum_{i=1}^{2^{k'-1}}\ell_i(z_{t+i})$
    }
    Draw $\hat i \in [|\L|]$ randomly with probability proportional to $\exp(-\alpha u_i)$\;
    Output model $\ell_{\hat i}$ for interval $z_{t+2^{k'-1}+1}, \ldots, z_{t+2^{k'}}$\;
\end{algorithm2e}

The following theorem states that this simple change in the decision rule indeed improves the dependence on $|\L|$ from $\poly(|\L|)$ to $\polylog(|\L|)$. We prove Theorem~\ref{thm:bounded-recall-upper} in Section~\ref{sec:bounded-recall}.

\begin{theorem}\label{thm:bounded-recall-upper}
    On model class $\L$ and an arbitrary sequence of $n$ data points, Algorithm~\ref{alg:exp-weights} with parameter $\alpha = \Theta(\sqrt{\log n\cdot\log|\L|})$ achieves an excess risk of $O\left(\sqrt{\frac{\log |\L|}{\log n}}\right)$ in expectation.
\end{theorem}

Although the excess risk of the above bounded-recall algorithm is exponentially suboptimal compared to Theorem~\ref{thm:hybrid}, it turns out to be nearly optimal (up to a square root) among all bounded-recall algorithms. In other words, to achieve the optimal rate in Theorem~\ref{thm:hybrid}, it is critical that the algorithm considers substantially more data points than the prediction window length $w$. The proof of the following theorem is deferred to Appendix~\ref{sec:missing-lower-bounded-recall}.

\begin{theorem}\label{thm:bounded-recall-lower}
    For any $n \ge 1$, $m \ge 2$ and any bounded-recall algorithm, there is an instance of selective learning with $|\L| = m$ and sequence length $n$ such that the expected excess risk of the algorithm is at least $\Omega\left(\min\left\{\frac{\log |\L|}{\log n}, 1\right\}\right)$.
\end{theorem}

\paragraph{Generalized mean prediction.} Our analysis of Algorithm~\ref{alg:exp-weights} builds on a new result for the generalized mean prediction setting, which is based on the following definition of boundedness and self-concordance.

\begin{definition}\label{def:bounded}
    Function $f: S \to \R$ is $C_0$-bounded if $\sup_{x \in S}f(x) - \inf_{x\in S}f(x) \le C_0$.
\end{definition}

\begin{definition}\label{def:self-concordance}
    Convex function $f: S \to \R$ is $C_1$-self-concordant if $f$ is three-times differentiable and the following holds for any $x, y \in S$: Let $g:[0, 1]\to\R$ be the restriction of $f$ to the line segment between $x$ and $y$, i.e., $g(t) \coloneqq f(x + t(y - x))$. Then, $|g'''(t)| \le C_1g''(t)$ for any $t \in [0, 1]$.
\end{definition}

Definition~\ref{def:self-concordance} differs from the usual definition of self-concordant functions (e.g., in~\cite{boyd2004convex}) in that the exponent on the second derivative is $1$ instead of $\frac{3}{2}$, and that our definition allows a general coefficient $C_1$. This makes the definition susceptible to affine transformations of the domain of $f$, i.e., scaling $S$ changes the parameter $C_1$. A similar definition of self-concordance was used by~\cite{tran2015composite} to analyze the convergence of gradient descent in convex optimization.

The algorithm that we analyze for generalized mean prediction is the same as the one in~\cite{drucker2013high,qiao2019theory}, which randomly chooses a timescale and a prediction window of that scale, and then makes a prediction based on the data points immediately before the prediction window. We formally state the algorithm as follows. Recall that $\U(S)$ denotes the uniform distribution over a finite set $S$.

\begin{algorithm2e}[H]
    \caption{Generalized Mean Prediction} \label{alg:mean-prediction}
    \KwIn{Online access to sequence $x \in S^n$.}
    $k \gets \lfloor\log_2 n\rfloor$; Draw $k' \sim  \U([k])$\;
    Draw $t \sim \U(\{0, 2^{k'}, 2\cdot 2^{k'}, 3\cdot 2^{k'}, \ldots, 2^k - 2^{k'}\})$\;
    Observe $x_1, x_2, \ldots, x_{t+2^{k'-1}}$\;
    $\hat x \gets \frac{1}{2^{k'-1}}\sum_{i=1}^{2^{k'-1}}x_{t+i}$\;
    Predict that the average of $x_{t+2^{k'-1}+1}, \ldots, x_{t+2^{k'}}$ equals $\hat x$\;
\end{algorithm2e}

\begin{theorem}\label{thm:mean-prediction}
    For any convex set $S \subseteq \R^d$ and convex function $f: S\to\R$ that is $C_0$-bounded and $C_1$-self-concordant, Algorithm~\ref{alg:mean-prediction} incurs an expected $f$-loss of $O\left(\frac{C_0(C_1 + 1)}{\log n}\right)$.
\end{theorem}

We prove Theorem~\ref{thm:mean-prediction} in Appendix~\ref{sec:mean} using an induction similar to the analyses in~\cite{drucker2013high,qiao2019theory}. We emphasize that Theorem~\ref{thm:mean-prediction} applies to any data sequence $x$, which might be chosen adversarially against the prediction algorithm. Specifically, the theorem makes \emph{no distributional assumptions} on the data, and the expectation is only over the randomness in the algorithm. Thus, this result is of the same flavor as those in prior work on selective prediction.

\begin{example}
    In the setting of Example~\ref{ex:mean-est}, since $f(x) = \frac{1}{2}x^2$ is $\frac{1}{2}$-bounded and $0$-self-concordant on $S=[0, 1]$, applying Theorem~\ref{thm:mean-prediction} recovers the $O\left(\frac{1}{\log n}\right)$ upper bound for predicting the arithmetic mean of a bounded number sequence.
\end{example}

\input{related}

\input{discuss}

%% file: related.tex
\subsection{Related Work}\label{sec:related}
    Most closely related to our work is the \emph{selective prediction} setting studied by~\cite{drucker2013high,qiao2019theory}. The setting models the prediction of certain statistics of the unseen data points in a data stream, rather than finding the best model that fits the unseen data. Specifically, \cite{drucker2013high} proved that given online access to an arbitrary binary sequence of length $n$, there is a predictor that selectively chooses a prediction window and predicts the average of the numbers in that window up to a squared error of $O(1 / \log n)$ in expectation. \cite{qiao2019theory} proved that this error is optimal up to constant factors, and also extended this positive result to more general functions beyond the arithmetic mean. The selective learning setting was first studied by~\cite{qiao2019theory} as an extension of  selective prediction. They proved that an algorithm based on empirical risk minimization (ERM) achieves an $O(\sqrt{|\L|/\log n})$ excess risk in expectation, and posed an open question regarding whether the dependence on $|\L|$ could be further improved.
    
    The exponential weights algorithm for the experts problem is due to \cite{littlestone1989weighted} and \cite{vovk1990aggregating}, and the idea has been extensively explored and generalized in subsequent work on online learning, e.g., \cite{freund1996predicting,freund1999adaptive}. The new algorithms proposed in this work differ from previous ones in that they disregard the data that are too far away from the prediction window, and compute the weights only with respect to the most recent data points at a similar timescale.
    
    Our work is part of a broader endeavor to understand the extent to which a learner could extract a non-trivial amount of information from non-i.i.d.\ and even worst-case data. This line of research dates back to at least the 1960s in the statistics literature. Early work along this line studied the possibility of being robust to a small amount of adversarial corruption under different contamination models and for various problems in estimation~\cite{huber1964robust} and learning~\cite{valiant1985learning,kearns1993learning}. There has been significant recent progress understanding the computational tractability of such robust estimation and learning problems; see e.g.~\cite{diakonikolas2016robust,lai2016agnostic,charikar2017learning,diakonikolas2018robustly,diakonikolas2019distribution} and the references therein. Note that in many of these works, the majority of the data are still assumed to be drawn i.i.d.\ from the underlying distribution, while the remainder could be arbitrary and adversarial.
    
    Another strand of literature focuses on time series data that come from processes with certain mixing properties. \cite{nobel1993note,yu1994rates} studied the uniform convergence property when the data sequence, though being dependent, comes from a stationary mixing process. \cite{agarwal2012generalization} studied the generalization guarantee of online learning algorithms trained on a data sequence generated by a mixing ergodic process.
    
    The recent work of~\cite{chen2020worst} introduced a framework for learning from worst-case data with only the knowledge of the process of partitioning $n$ data points into a training set and a test set. This framework incorporates the positive results on selective prediction as a special case, where the partition follows a chronological constraint. The authors provided an algorithm that computes a near-optimal prediction scheme that matches the optimal error up to a constant factor. Another recent work of~\cite{dagan2019learning} studied the problem of learning from weakly dependent data that satisfy Dobrushin's condition, and proved that the learning guarantees only degrade slightly compared to the i.i.d.\ setting.

%% file: discuss.tex
\subsection{Discussion}\label{sec:discuss}

\paragraph{Further open problems.} One of the most obvious problems that we leave open is to close the gap between the bounds in Theorems \ref{thm:hybrid}~and~\ref{thm:learn-lower}. One possible starting point is the case that $|\L| = 2$, in which case   Theorem~\ref{thm:hybrid} gives an upper bound of $O(\frac{\log\log n}{\log n})$; if the learner can be adaptive, the current lower bound is vacuous since $|\L| \ge n^{\omega(\log\log n)}$ does not hold.  Even if we restrict ourselves to non-adaptive algorithms, there is still a $\log\log n$ multiplicative gap between $O(\frac{\log\log n}{\log n})$ and the $\Omega(\frac{1}{\log n})$ lower bound.

One aspect of selective learning that is not addressed by this work is its computational complexity. When the model class is exponentially large in $n$, straightforward implementations of Algorithms \ref{alg:exp-weights}~and~\ref{alg:hybrid-exp-weights} are computationally costly. Nevertheless, since both algorithms are based on exponential weights, they can be efficiently implemented as long as it is possible to efficiently sample from class $\mathcal{L}$ (according to the distribution defined by the data points). This efficient sampling is possible if $\mathcal{L}$ has certain structure, for example: (1) if $\L$ can be decomposed into the ``product'' of smaller concept classes; or (2) if $\mathcal{L} = \{\ell_w: w \in \mathcal{W}\}$ can be parametrized such that for every data point $z$, the loss $\ell_w(z)$ is convex in the parameter $w$ (e.g., linear regression under quadratic loss).

Even if no efficient sampling algorithms are unknown for class $\L$, it \emph{might} still be possible to achieve a small excess risk that is comparable to Algorithm~\ref{alg:hybrid-exp-weights} using a different algorithm that admits efficient implementations. We leave the exploration of this possibility as a direction for future work.

\paragraph{Cardinality vs VC dimension.} In learning theory, many positive results that hold for finite model classes can be extended to the infinite case, with the $\log|\L|$ term in the sample complexity replaced by certain complexity measure of the class, e.g., the VC dimension~\cite{vapnik1971uniform}. Given Theorem~\ref{thm:hybrid}, it is natural to ask whether the $\log\log|\L|$ term can be replaced by $\log d$ in general, where $d$ is the VC dimension of $\L$. More formally, we consider the binary classification case where $\Y = \{0, 1\}$ and $\ell(\hat y, y) = \I{\hat y \ne y}$ is the binary loss, and ask whether the optimal excess risk scales as $\frac{\log d}{\log n}$.

It turns out that this is not the case: there exist model classes with a constant VC dimension that cannot be selectively learned to a sub-constant excess risk for arbitrarily large $n$. Formally, we prove the following proposition in Appendix~\ref{sec:missing-discuss}:

\begin{proposition}\label{prop:vc-not-upper}
    There exists a model class $\F$ of VC dimension $1$ such that for any sequence length $n$, no learning algorithm could achieve an excess risk smaller than $\frac{1}{2}$ in expectation for class $\F$.
\end{proposition}

\paragraph{Realizable vs agnostic settings.} In PAC learning, the realizable setting refers to the special case where the model class $\F$ is guaranteed to contain a model that is consistent with all the data (i.e., the model achieves a loss of zero on the data distribution), while the more general setting where a perfect model may not exist is called the agnostic case. The sample complexities for learning a model with excess risk $\eps$ (e.g., with probability $0.99$) are the same (up to constant factors) for the realizable and agnostic settings when $\eps = \Omega(1)$---both are linear in the VC dimension of the model class. In contrast, this is not the case for the selective learning setting; when the data sequence is guaranteed to be consistent with some model in $\L$, there exists a simple algorithm that achieves an $O(\log|\L|/n)$ excess risk. This gives an exponential improvement in terms of the dependence on $n$ over the bounds for the general agnostic setting that scale as $1/\polylog(n)$. We prove the following proposition in Appendix~\ref{sec:missing-discuss}.

\begin{proposition}\label{prop:realizable}
    For the selective learning problem, under the promise that there exists a model $\ell^*\in\L$ such that $\ell^{*}(z_i) = 0$ for every $i \in [n]$, there is an algorithm with expected excess risk of $O\left(\frac{\log|\L|}{n}\right)$.
\end{proposition}

%% file: hybrid.tex
\section{The Hybrid Exponential Weights Algorithm}\label{sec:hybrid}
We introduce a new algorithm for selective learning, termed the \emph{hybrid exponential weights} algorithm, and prove Theorem~\ref{thm:hybrid}. Recall that $\U(S)$ denotes the uniform distribution over set $S$.

\begin{algorithm2e}[H]
    \caption{Hybrid Exponential Weights} \label{alg:hybrid-exp-weights}
    \KwIn{Model class $\L=\{\ell_1,\ell_2,\ldots,\ell_{|\L|}\}$ and online access to data sequence $z \in \Z^n$. Parameters $\Delta \in \{1, 2, \ldots, \lfloor\log_2n\rfloor\}, \eta > 0$.}
    $k \gets \lfloor\log_2 n\rfloor$; Draw $w \sim  \U(\{2^0, 2^1, \ldots, 2^{k - \Delta}\})$;
    $W \gets 2^\Delta \cdot w$\;
    Draw $i_1 \sim \U([2^k/W])$;
    $t_0 \gets (i_1 - 1)\cdot W$\;
    Draw $i_2 \sim \U([2^{\Delta}])$;
    $t \gets t_0 + (i_2 - 1) \cdot w$\;
    Observe $z_1, z_2, \ldots, z_t$\;
    \lFor{$i = 1, 2, \ldots, |\L|$} {
	    $u_i \gets \frac{1}{w}\sum_{j=t_0 + 1}^{t}\ell_i(z_j)$
    }
    Draw $\hat i \in [|\L|]$ randomly with probability proportional to $\exp(-\eta \cdot u_i)$\;
    Output model $\ell_{\hat i}$ for interval $z_{t+1}, \ldots, z_{t+w}$\;
\end{algorithm2e}

The algorithm first chooses the prediction window length $w$ uniformly at random from all powers of two between $1$ and $n/2^\Delta$, and also considers another window length $W$ that is $2^\Delta$ times larger. The actual prediction window is selected by first picking a longer window of length $W$ that starts at $t_0 + 1$, and then randomly choosing one of the $2^\Delta$ shorter windows of length $w$ within the longer window. The final choice of the model depends solely on the data points within the longer window, and the probability of choosing a model is negatively exponential in the cumulative loss of the model within the longer window.

Note that the hybrid exponential weights algorithm is not bounded-recall in the sense of Definition~\ref{def:learner}. For example, when $i_2 = 2^\Delta$ is chosen, the choice of the model $\ell_{\hat i}$ depends on the most recent $t - t_0 = (2^\Delta - 1)\cdot w$ data points, whereas the prediction window only contains $w$ data points. This difference from the ERM algorithm and Algorithm~\ref{alg:exp-weights} is crucial for circumventing the lower bound result against bounded-recall learners in Theorem~\ref{thm:bounded-recall-lower}.

The following lemma defines a sequence of quantities $L_0, L_1, \ldots, L_k$ such that $L_i$ represents the average learnability of the data sequence at timescale $2^i$. More formally, $L_i$ is defined by partitioning the first $2^k$ data points into $2^{k-i}$ disjoint blocks of size $2^i$, and then calculating the expected minimum loss (among model class $\L$) over a randomly chosen block. Then, the lemma relates the expected excess risk incurred by the algorithm conditioning on $w = 2^i$ to the gap between $L_i$ and $L_{i+\Delta}$.

\begin{lemma}\label{lem:conditional}
    Fix a model class $\L$ and a data sequence $z \in \Z^n$. Let $k = \lfloor\log_2 n\rfloor$. For $i \in \{0, 1, \ldots, k\}$, define
        $L_i \coloneqq
    \frac{1}{2^{k-i}}\sum_{j=1}^{2^{k-i}}\min_{\ell \in \L}\frac{1}{2^i}\sum_{j'=1}^{2^i}\ell\left(z_{(j - 1)\cdot 2^i + j'}\right)$.
    Then, the expected excess risk incurred by Algorithm~\ref{alg:hybrid-exp-weights} conditioning on $w = 2^i$ is at most
        $(L_{i+\Delta} - L_i) + \frac{\ln|\L|}{\eta \cdot 2^{\Delta}} + \frac{\eta}{8}$.
    In particular, for $\eta = \sqrt{8\ln|\L|/2^{\Delta}}$, the above is
        $(L_{i+\Delta} - L_i) + \sqrt{\ln|\L|/2^{\Delta + 1}}$.
\end{lemma}

The proof of the lemma relies on the following observation: conditioning on the choice of $w$ and $t_0$, the behavior of Algorithm~\ref{alg:hybrid-exp-weights} is the same as that of the classic exponential weights algorithm (e.g.,~\cite[Figure 8.5]{mohri2018foundations}) on a hypothetical instance of the experts problem with $2^\Delta$ time steps and $|\L|$ experts. Each single time step in the experts problem instance corresponds to $w$ contiguous data points in the original sequence. This justifies the name ``hybrid exponential weights'': the algorithm essentially simulates the exponential weights algorithm at a randomly chosen timescale $w$.

\begin{proof}
    Let random variable $X$ denote the average loss of the model chosen by Algorithm~\ref{alg:hybrid-exp-weights} over the prediction window, and let $Y$ be the minimum possible average loss among all models in $\L$ over the window. The expected excess risk is then $\Ex{X - Y} = \Ex{X} - \Ex{Y}$. We claim that for each $i \in \{0, 1, \ldots, k - \Delta\}$: (1) $\Ex{Y | w = 2^i} = L_i$; (2) $\Ex{X | w = 2^i} \le L_{i+\Delta} + \frac{\ln|\L|}{\eta \cdot 2^{\Delta}} + \frac{\eta}{8}$. The lemma follows immediately from the two claims.
    
    \paragraph{Proof of claim (1).} Conditioning on $w = 2^i$, the prediction window spans $w$ time steps: $t + 1, t + 2, \ldots, t + w$, and the stopping time $t$ is chosen as $t = (i_1 - 1)\cdot W + (i_2 - 1)\cdot w$. By our choice of $i_1 \sim \U([2^k/W])$ and $i_2 \sim \U([2^{\Delta}])$, $t$ is uniformly distributed over $\{0, w, 2w, \ldots, 2^k - w\}$. In other words, the prediction window is uniformly distributed over the $2^k/w$ disjoint blocks of size $w$ formed by the first $2^k$ time steps. Therefore, the conditional expectation of $Y$ is exactly $L_i$.
    
    \paragraph{Proof of claim (2).} Recall that the algorithm draws $i_1 \sim \U([2^k/W])$ and let $t_0 = (i_1 - 1) \cdot W$. We fix the value of $i_1$ (and thus $t_0$) and divide the data sequence $z_{t_0 + 1}, z_{t_0 + 2}, \ldots, z_{t_0 + W}$ into $2^\Delta$ blocks of equal size $w = 2^i$.
    
    Consider a hypothetical instance of the experts problem with time horizon $T = 2^\Delta$ and $N = |\L|$ experts. The loss incurred by expert $j \in [N]$ at time $t' \in [T]$ is defined as $\frac{1}{w}\sum_{j'=1}^{w}\ell_j(z_{t_0 + (t' - 1)\cdot w + j'})$, i.e., the average loss of model $\ell_j$ over the $t'$-th block. Suppose that we run the exponential weights algorithm with parameter $\eta$. Then, the weight of the $j$-th expert at time $t'$ is equal to
    \[
        \exp\left(-\eta\cdot\sum_{t'' = 1}^{t' - 1}\frac{1}{w}\sum_{j'=1}^{w}\ell_j(z_{t_0 + (t'' - 1)\cdot w + j'})\right)
    =   \exp\left(-\eta\cdot\frac{1}{w}\sum_{j'=t_0+1}^{t_0+(t' - 1)\cdot w}\ell_j(z_{j'})\right).
    \]
    Note that this is proportional to the distribution from which Algorithm~\ref{alg:hybrid-exp-weights} chooses the model when $i_2 = t'$. Therefore, the expected loss of the chosen model conditioning on $w = 2^i$ and the value of $i_1$ is exactly the same as the expected average loss incurred by the exponential weights algorithm in the hypothetical instance over the $2^\Delta$ steps. The performance guarantee of the exponential weights algorithm (e.g., \cite[Theorem 8.6]{mohri2018foundations}) then implies that this loss is at most
    \[
        \min_{\ell \in \L}\frac{1}{W}\sum_{j'=1}^{W}\ell\left(z_{t_0 + j'}\right) + \frac{\ln N}{\eta T} + \frac{\eta}{8}
    =   \min_{\ell \in \L}\frac{1}{W}\sum_{j'=1}^{W}\ell\left(z_{t_0 + j'}\right) + \frac{\ln|\L|}{\eta \cdot 2^\Delta} + \frac{\eta}{8}.
    \]
    Finally, note that $t_0$ is uniformly distributed over $\{0, W, \ldots, 2^k - W\}$. Taking the expectation over the random choice of $i_1$, the first term above becomes $L_{i + \Delta}$, which proves the claim.
\end{proof}

Theorem~\ref{thm:hybrid} follows immediately from Lemma~\ref{lem:conditional}: Since $w$ is chosen uniformly at random from $\{2^0, 2^1, \ldots, 2^{k - \Delta}\}$, the gap $L_{i+\Delta} - L_i$ is small in expectation. Furthermore, we can balance the expected gap and the regret term $\sqrt{\ln|\L|/2^{\Delta + 1}}$ by choosing $\Delta$ appropriately.

\begin{proof}[Proof of Theorem~\ref{thm:hybrid}]
    Since Algorithm~\ref{alg:hybrid-exp-weights} draws $w$ uniformly at random from $\{2^0, 2^1, \ldots, 2^{k - \Delta}\}$, by Lemma~\ref{lem:conditional} and the law of total expectation, the expected excess risk of the algorithm is at most
    \begin{align*}
        &~~\frac{1}{k - \Delta + 1}\sum_{i=0}^{k-\Delta}\left[(L_{i+\Delta} - L_i) + \sqrt{\ln|\L|/2^{\Delta + 1}}\right]\\
    =   &~~\frac{\sum_{i=0}^{\Delta - 1}L_i - \sum_{i=k - \Delta + 1}^{k}L_i}{k - \Delta + 1} + \sqrt{\frac{\ln|\L|}{2^{\Delta + 1}}}
    =   O\left(\frac{\Delta}{\log n} + \sqrt{\frac{\log|\L|}{2^\Delta}}\right)
    \end{align*}
    assuming that $\eta = \sqrt{8\ln|\L|/2^{\Delta}}$. Here the second step follows from the observation that $L_i \in [0, 1]$ and $k = \Theta(\log n)$. Therefore, for some $\Delta = \Theta(\log\log|\L| + \log\log n)$, the hybrid exponential weights algorithm achieves the desired upper bound.
\end{proof}

%% file: bounded-recall.tex
\section{A Near-Optimal Bounded-Recall Algorithm}\label{sec:bounded-recall}
    In this section, we analyze the excess risk of the bounded-recall exponential weights algorithm (Algorithm~\ref{alg:exp-weights}) and prove Theorem~\ref{thm:bounded-recall-upper}. Recall that Algorithm~\ref{alg:exp-weights} is associated with a parameter $\alpha > 0$ that controls the exponential weights of the models.
    The proof starts by decomposing the excess risk into three parts.
    At an intuitive level, the first term can be regarded as a ``generalization gap'', which increases as $\alpha$ gets larger (i.e., the algorithm becomes more aggressive). The key analytical idea in our proof is that because of the choice of probability distribution induced by exponential weights, this term is closely related to a Bregman divergence, and can be analyzed using Theorem~\ref{thm:mean-prediction}.
    The second term measures the optimization error, which would be large if the algorithm is overly conservative (i.e., $\alpha$ is too small). Its analysis involves noting that the exponential probabilities from the algorithm induce a sort of subexponential tail phenomenon, so this term behaves like the maximum of $|\L|$ subexponential random variables.
    The last term can be upper bounded independently of $\alpha$.
    We will upper bound the expectation of each part by a function of $\alpha$, and plug in the optimal choice of $\alpha$ in the end.
    
    \begin{proof}[Proof of Theorem~\ref{thm:bounded-recall-upper}]
    Fix a model class $\L$ and a data sequence $z\in \Z^n$. Define $u, v \in [0, 1]^{|\L|}$ as follows:
    \[
        u_i \coloneqq \frac{1}{2^{k'-1}}\sum_{i=1}^{2^{k'-1}}\ell_i(z_{t+i}), \quad
        v_i \coloneqq \frac{1}{2^{k'-1}}\sum_{i=1}^{2^{k'-1}}\ell_i(z_{t+2^{k'-1}+i}),
    \]
    i.e., $u$ denotes the average losses of the models on the observed data sequence immediately before the prediction window (``empirical error''), while $v$ denotes the average losses on the prediction window (``test error''). Let $P_u \in \R^{|\L|}$ denote the probability distribution defined by $u$ and $\alpha > 0$, i.e., $P_u(i) = e^{-\alpha u_i}/\left(\sum_{j=1}^{|\L|}e^{-\alpha u_j}\right)$.
    Let $v_{\min} \coloneqq \min_{i \in [|\L|]}v_i$ denote the minimum possible loss on the prediction window, and define $u_{\min}$ similarly. Note that here $u$, $v$, $u_{\min}$, $v_{\min}$ are all random variables induced by the randomness in the algorithm.
    
    The excess risk of the learner can be decomposed into three terms as follows:
    \[
        \sum_{i=1}^{|\L|}P_u(i)(v_i - v_{\min})
    =  \sum_{i=1}^{|\L|}P_u(i)(v_i - u_i) + \sum_{i=1}^{|\L|}P_u(i)(u_i - u_{\min}) + \sum_{i=1}^{|\L|}P_u(i)(u_{\min} - v_{\min}).
    \]
    
    \paragraph{The first term.} Consider the log-sum-exp function $f(x) \coloneqq \ln\left(\sum_{i=1}^{|\L|}e^{-\alpha x_i}\right)$. Note that since each $x_i \in [0, 1]$, $f(x)$ takes value between $\ln|\L| - \alpha$ and $\ln|\L|$ and is thus $\alpha$-bounded. It is well-known that $f$ is convex, and its gradient is given by $\nabla f(x) = -\alpha P_x$, so the Bregman divergence $D_f(v, u)$ can be written as
    \[
        D_f(v, u)
    =   f(v) - f(u) - \nabla f(u)^{\top}(v-u)
    =   f(v) - f(u) + \alpha\sum_{i=1}^{|\L|}P_u(i)(v_i - u_i).
    \]
    Thus, the first term is upper bounded by
    \begin{equation}\label{eq:first-term-decomp}
        \sum_{i=1}^{|\L|}P_u(i)(v_i - u_i)
    =   \frac{D_f(v, u) + f(u) - f(v)}{\alpha}
    \le \frac{D_f(u, v) + D_f(v, u) + |f(u) - f(v)|}{\alpha},
    \end{equation}
    where the second step follows from the non-negativity of Bregman divergence.
    We prove in Appendix~\ref{sec:missing-bounded-recall} that $f$ is $(4\alpha)$-self-concordant as defined in Definition~\ref{def:self-concordance}, so Theorem~\ref{thm:mean-prediction} implies
        $\Ex{D_f(u, v) + D_f(v, u)} \le O\left(\frac{\alpha(\alpha + 1)}{\log n}\right)$.
    To upper bound the expectation of $|f(u) - f(v)|$, we consider a family of functions $f_1, f_2, \ldots, f_n$ where $f_m:\Z^w\to\R$ is defined as $f_m(\hat z_1, \hat z_2, \ldots, \hat z_m) = -f(x)$ for $x \in \R^{|\L|}$ where $x_i = \frac{1}{w}\sum_{j=1}^{w}\ell_i(\hat z_j)$. In words, $f_m$ maps $w$ labeled data points to the negated function value of $f$ on the loss vector defined by the models and the $w$ labeled examples.
    Since $f$ is convex, it is easy to verify that $(f_1, f_2, \ldots, f_m)$ is concatenation-concave in the sense of Definition 5 in~\cite{qiao2019theory}. Thus, as $f$ takes value in an interval of length $\alpha$, after shifting and scaling $f$ by $1/\alpha$, Theorem 6 in~\cite{qiao2019theory} implies that
        $\Ex{|f(u) - f(v)|} \le \sqrt{\Ex{(f(u) - f(v))^2}} \le O\left(\frac{\alpha}{\sqrt{\log n}}\right)$.
    Therefore, it follows from \eqref{eq:first-term-decomp} that
    \[
        \Ex{\sum_{i=1}^{|\L|}P_u(i)(v_i - u_i)}
    \le \frac{1}{\alpha}\cdot O\left(\frac{\alpha(\alpha + 1)}{\log n} + \frac{\alpha}{\sqrt{\log n}}\right)
    = O\left(\frac{\alpha}{\log n} + \frac{1}{\sqrt{\log n}}\right).
    \]

    \paragraph{The second term.} Note that the second term only depends on $u$. We will show that for \emph{any} $u \in \R^{|\L|}$, the term is at most $O\left(\frac{\log|\L|}{\alpha}\right)$. The same bound then holds for any distribution over $u$. For fixed $u \in \R^{|\L|}$, let $I = \left\{i \in [|\L|]: u_i - u_{\min} \le \frac{2\ln|\L|}{\alpha}\right\}$ and $J = [|\L|]\setminus I$. Then, we have
    \[
        \sum_{i\in I}P_u(i)(u_i - u_{\min})
    \le \sum_{i\in I}P_u(i)\cdot \frac{2\ln|\L|}{\alpha}
    \le \frac{2\ln|\L|}{\alpha}.
    \]
    Moreover, for each $i \in J$, we have
    \[
    P_u(i)(u_i - u_{\min}) \le \frac{e^{-\alpha u_i}}{e^{-\alpha u_{\min}}}\cdot (u_i - u_{\min}) \le e^{-\alpha\cdot 2\ln|\L|/\alpha}\cdot \frac{2\ln|\L|}{\alpha} = \frac{2\ln|\L|}{\alpha|\L|^2}.
    \]
    The second step follows from the monotonicity of $x\mapsto e^{-\alpha x}\cdot x$ on $(1/\alpha, +\infty)$ and that $2\ln|\L|/\alpha \ge 1/\alpha$ for $|\L| \ge 2$. Putting the two parts together, we obtain
    \[
        \sum_{i=1}^{|\L|}P_u(i)(u_i - u_{\min})
    \le \frac{2\ln|\L|}{\alpha} + |\L|\cdot\frac{2\ln|\L|}{\alpha|\L|^2}
    \le O\left(\frac{\log|\L|}{\alpha}\right).
    \]

    \paragraph{The third term.} The last term reduces to $u_{\min} - v_{\min}$, which is upper bounded by $|u_{\min} - v_{\min}|$. This is exactly the absolute loss for predicting the \emph{learnability function} in selective prediction setting. Since learnability is concatenation-concave and takes value in $[0, 1]$, it follows from Theorem 6 in~\cite{qiao2019theory} and Jensen's inequality that
    \[
        \Ex{u_{\min} - v_{\min}}
    \le \sqrt{\Ex{(u_{\min} - v_{\min})^2}}
    \le O\left(\frac{1}{\sqrt{\log n}}\right).
    \]
    
\noindent    Combining the three terms and setting $\alpha = \Theta\left(\sqrt{\log n\cdot\log|\L|}\right)$ yields the upper bound.
    \end{proof}

%% file: missing-proofs.tex
\section{Proof of the General Lower Bound}\label{sec:missing-lower}
    We start with the following lemma, which states that if the length of the prediction window is restricted to a short range, the learner has to incur an $\Omega(1)$ excess risk in expectation.
    
    \begin{lemma}\label{lem:interval}
        Let $n, m, l, r$ be integers that satisfy $1 \le l \le r \le n$ and $r/l = O(\log m)$.
        For any selective learning algorithm that always picks a window length in $[l, r]$, there exists an instance with sequence length $n$ and $|\L| = m$ on which the learner incurs an $\Omega(1)$ excess risk.
    \end{lemma}

    \begin{remark}\label{rem:construct}
        Before proving Lemma~\ref{lem:interval}, we remark that to construct an instance of selective learning with $m$ models and $n$ data points, it suffices to specify $m$ binary sequences $a^{(1)}, a^{(2)}, \ldots, a^{(m)} \in \{0, 1\}^n$. This is because we can always construct an instance as follows to ensure $\ell_i(z_j) = a^{(i)}_j$ for each $(i, j) \in [m]\times[n]$:
        \begin{itemize}
            \item The instance space and label space are $\X = \{0, 1\}^m$ and $\Y = \{0, 1\}$. 
            \item The loss function is the binary loss $\ell(\hat y, y) = \I{\hat y\ne y}$.
            \item The model class is $\F = \{f_1, f_2, \ldots, f_m\}$ where $f_i(x) = x_i$ for any $x \in \X$.
            \item The $j$-th data point $(x_j, y_j)$ is defined such that the $i$-th bit of $x_j$ is $a^{(i)}_j$, and $y_j = 0$.
        \end{itemize}
        Then, the loss of $f_i$ on $(x_j, y_j)$ is $\ell(f_i(x_j), y_j) = f_i(x_j)$, i.e., the $i$-th bit of $x_j$, which is exactly $a^{(i)}_j$.
    \end{remark}

    \begin{proof}[Proof of Lemma~\ref{lem:interval}]
        We will consider a distribution over instances, and show that the algorithm incurs a constant excess risk on a random instance from the distribution. Then, the lemma follows from an averaging argument. In light of Remark~\ref{rem:construct}, we construct the instance by randomly generating $m$ sequences independently as follows. We partition the time horizon $1, 2, \ldots, n$ into $4n/l$ blocks of length $l/4$. Each block consists of $l/4$ copies of the same random bit drawn from $\{0, 1\}$ uniformly and independently. Equivalently, each binary sequence is constructed by sampling a random sequence from $\{0, 1\}^{4n/l}$ and duplicating each bit $(l/4)$ times.
        
        Now we analyze the excess risk of the algorithm on a randomly generated instance. Suppose that the learner picks a window length $w \in [l, r]$ and a model $\ell_{\hat i}$. We say that the learner has observed a block if it has seen at least one data point in that block. Since each block is of length $l/4$ and $w \ge l$, the prediction window contains at most $l/4-1 \le w/4$ data points in the observed blocks. By construction, the loss of $\ell_{\hat i}$ on each of the other $\ge 3w/4$ data points is a uniformly random bit, even after conditioning on the observations so far. Thus, the average loss of $\ell_{\hat i}$ over the prediction window is at least $\frac{1}{2}\cdot\frac{3w/4}{w} = \frac{3}{8}$ in expectation.
        
        On the other hand, there are at most $w / (l/4) \le 4r/l$ unobserved blocks in the prediction window. Thus, for each of the $m - 1$ models other than $\ell_{\hat i}$, with probability at least $2^{-4r/l}$, the losses of the model on those unobserved blocks are zero, which implies that its average loss over the prediction window is at most $(l/4 - 1)/w \le \frac{1}{4}$. Since we assumed that $r/l = O(\log m)$, as long as the constant in $O(\cdot)$ is sufficiently small, it holds that $2^{-4r/l} \ge \frac{1}{m-1}$. Since the sequences are independent, with probability at least $1 - (1 - \frac{1}{m-1})^{m-1} \ge 1 - e^{-1} = \Omega(1)$, at least one of the other $m - 1$ models satisfy the condition above. Conditioning on this, the excess risk of the chosen model $\ell_{\hat i}$ is at least $\frac{3}{8} - \frac{1}{4} = \Omega(1)$. This proves of the $\Omega(1)$ lower bound.
    \end{proof}
    
    Using Lemma~\ref{lem:interval}, we prove the first case of Theorem~\ref{thm:learn-lower} (in which the algorithm is non-adaptive) by constructing multiple hard instances for different sub-intervals of $[1, n]$.
    \begin{proof}[Proof of Theorem~\ref{thm:learn-lower}, Case (1)]
        Fix some parameter $c = \Theta(\log m)$ with the same hidden constant as in Lemma~\ref{lem:interval}. Consider the $k = \lceil\log_c n\rceil$ intervals $[1, c], [c, c^2], \ldots, [c^{k-1}, c^k]$. Since the algorithm is non-adaptive and always picks a window length in $[1, n]$, it can be viewed as the mixture of $k$ sub-algorithms, where the $i$-th sub-algorithm has a probability of $p_i$ and always chooses a window length in $[c^{i-1}, c^{i}]$. More formally, the non-adaptive algorithm is equivalent to randomly picking $i\in[k]$ according to $p_1, p_2, \ldots, p_k$ and then running the $i$-th sub-algorithm. 
        
        Note that there exists $i\in[k]$ such that $p_i \ge \frac{1}{k} = \frac{1}{\lceil\log_c n\rceil} = \Omega\left(\min\left\{\frac{\log\log m}{\log n}, 1\right\}\right)$. By Lemma~\ref{lem:interval}, there exists an instance on which the $i$-th sub-algorithm incurs a constant loss. Therefore, the expected loss of the original algorithm on the same instance is at least 
            \[p_i\cdot\Omega(1) = \Omega\left(\min\left\{\frac{\log\log m}{\log n}, 1\right\}\right).\]
    \end{proof}
    
    The second case of Theorem~\ref{thm:learn-lower} is proved by constructing $|\L|$ independent sequences of losses, such that each loss sequence is sufficiently anti-concentrated at all timescales. This can be accomplished using a variant of the lower bound construction in~\cite[Section 2.3]{qiao2019theory}: Assume that $n = 2^k$ for some integer $k$ and consider a full binary tree $\T$ with $2^k$ leaves and $k + 1$ levels numbered $0, 1, \ldots, k$ from top to bottom. We label each node of $\T$ with a value in $[0, 1]$ using the following procedure:
    \begin{itemize}
        \item The value of a node in level $j$ is either $\frac{1}{2} + \frac{j}{4k}$ or $\frac{1}{2} - \frac{j}{4k}$. In particular, the root is labeled $\frac{1}{2}$.
        \item For each node $v$ with parent $u$, conditioning on the value of $u$, the expected value of $v$ is equal to the value of $u$.
    \end{itemize}
    Equivalently, conditioning on that the parent is labeled $\frac{1}{2} + \frac{j-1}{4k}$, the child takes value $\frac{1}{2} + \frac{j}{4k}$ with probability $1 - \frac{1}{2j}$, and $\frac{1}{2} - \frac{j}{4k}$ with probability $\frac{1}{2j}$; the two probabilities are switched if the parent has value $\frac{1}{2} - \frac{j-1}{4k}$. Note that the values on every root-to-leaf path in $\T$ form a martingale.
    
    For each edge between a node $v$ and its parent $u$, we say that the edge is a \emph{flip}, if $u$ and $v$ are assigned values that are on the opposite sides of $1/2$. Then, if node $v$ is in the $j$-th level, the probability of a flip is exactly $\frac{1}{2j} \in [\frac{1}{2k}, \frac{1}{2}]$.
    
    Given a randomly labeled tree $\T$ generated as above, we define a binary sequence of length $n = 2^k$ using $\T$ as follows. We first number the leaves of $\T$ with $1, 2, \ldots, 2^k$ in the natural way, such that for every internal node, every leaf in its left subtree has a smaller number than every leaf in the right subtree does. For each $i \in [2^k]$, we look at the value $p$ of the leaf with number $i$, and the $i$-th bit of the sequence will be an independent sample from the Bernoulli distribution with parameter $p$. The binary sequence is obtained by concatenating these $n$ independent bits. Let $\D_n$ denote the probability distribution over $\{0, 1\}^n$ that is implicitly defined by the above procedure.
    
    Now we are ready to prove the second case of Theorem~\ref{thm:learn-lower}.
    
    \begin{proof}[Proof of Theorem~\ref{thm:learn-lower}, Case (2)]
        Let $a^{(1)}, a^{(2)}, \ldots, a^{(m)}$ be $m$ independent samples from $\D_n$, and let $\T_i$ denote the labeled binary tree that corresponds to $a^{(i)}$. By Remark~\ref{rem:construct}, we can construct a model class $\L = \{\ell_1, \ldots, \ell_m\}$ and data sequence $z \in \Z^n$ such that $\ell_i(z_j)$ is equal to $a^{(i)}_j$, the $j$-th bit of the $i$-th string.
        
        Suppose that a selective learning algorithm, when running on the instance defined as above, stops after $t$ steps and chooses model $\ell_{\hat i}$ for the next $w$ data points. We will prove the following stronger result: conditioning on every possible triple $(t, w, \hat i)$, the conditional expectation of the excess risk incurred by the algorithm is at least $\Omega\left(\min\left\{\frac{\log\log|\L|}{\log n}, 1\right\}\right)$. In the following, we will assume that $t \ge 1$; the case that $t = 0$ can be handled similarly.
        
        \paragraph{Observed nodes and critical nodes.} For each node in each of the full binary trees $\T_1, \T_2, \ldots, \T_m$, we say that the node is \emph{observed} if at least one of the leaves among its descendants has a number at most $t$; otherwise the node is \emph{unobserved}. Intuitively, given the first $t$ entries of $a^{(i)}$, the learner may have a good knowledge about the value of each observed node in $\T_i$, whereas the value of each unobserved node is still sufficiently random.
        
        Furthermore, we say that a node is \emph{critical} if the following two conditions hold: (1) the node itself is unobserved but its parent is observed; (2) among the descendants of the node, there is a leaf numbered between $t + 1$ and $t + w$. The first condition guarantees that the value of each critical node has a decent chance of dropping below $1/2$, even conditioning on the observations in the first $t$ time steps. The second condition guarantees that the value of each critical node affects the average of the sequence within the prediction window.
        
        \paragraph{Decompose the prediction window.} The prediction window corresponds to $w$ consecutive leaves numbered $t + 1$ through $t + w$. Since all these $w$ leaves are unobserved and the root is observed as long as $t \ge 1$, each of the $w$ leaves has a unique ancestor that is critical. We can group these $w$ leaves based on their critical ancestors: let $v_1, v_2, \ldots, v_q$ be the $q$ critical nodes sorted in the ascending order of their levels. Let $d_i$ be the level of $v_i$, and let $n_i$ be the number of leaves numbered between $t+1$ and $t+w$ in the subtree rooted at $v_i$. See Figure~\ref{fig:decomp} for an example of this decomposition.
        
        \begin{figure}[h]
        \centering
        \includegraphics[scale=0.2]{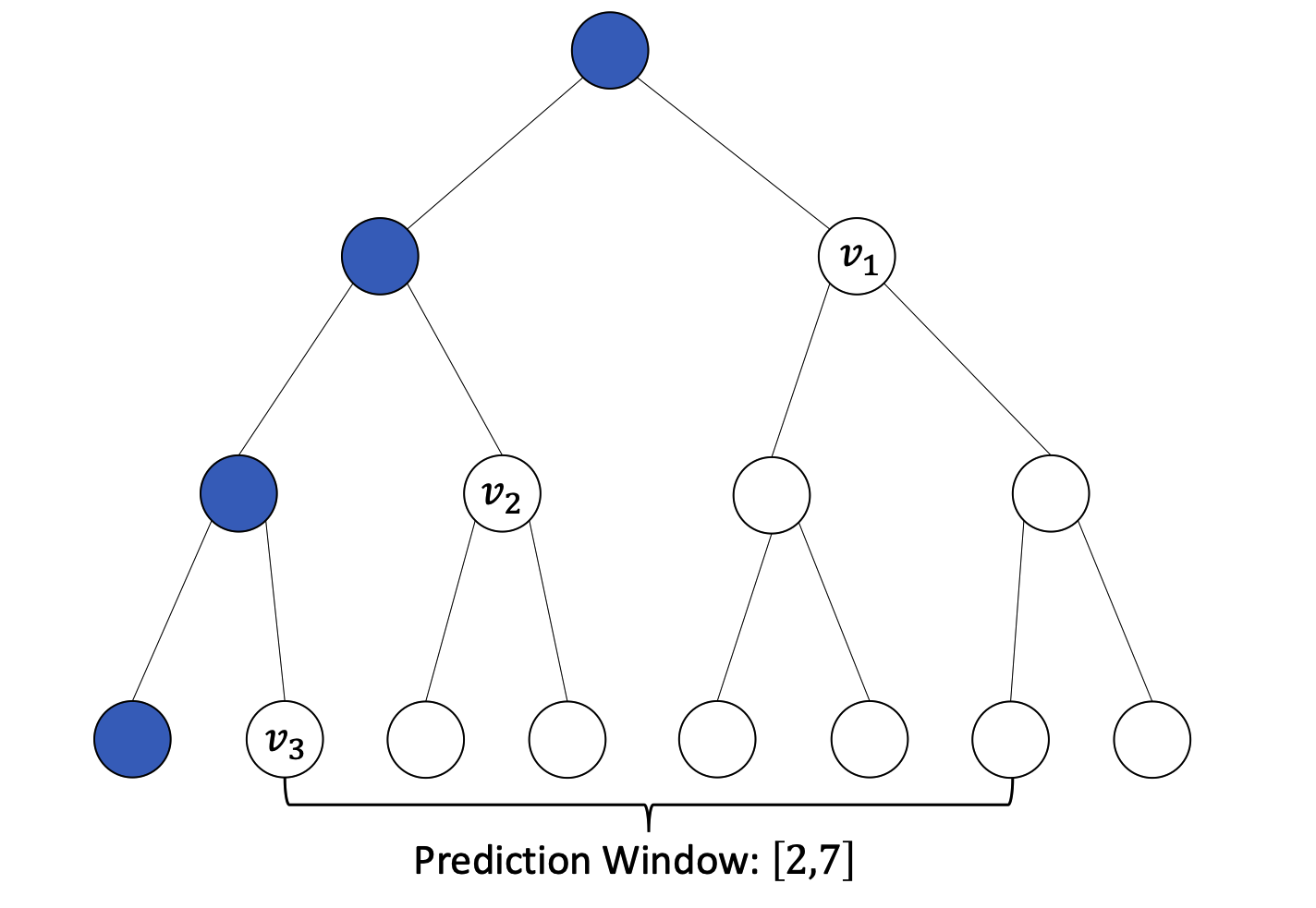}
        \caption{An example of the decomposition with $k = 3$, $t = 1$ and $w = 6$. The shaded nodes are observed. $v_1, v_2, v_3$ are the critical nodes, with $(d_1, d_2, d_3) = (1, 2, 3)$ and $(n_1, n_2, n_3) = (3, 2, 1)$.}
        \label{fig:decomp}
        \end{figure}
        
        We have the following observations:
        \begin{itemize}
            \item No two critical nodes are in the same level, so $1 \le d_1 < d_2 < \cdots < d_q \le k$ and thus $q \le k$.
            \item For each $i \ge 2$, $n_i = 2^{k - d_i}$. In other words, every leaf in the subtrees rooted at $v_2, v_3, \ldots, v_q$ is numbered between $t + 1$ and $t + w$.
            \item $n_1 + n_2 + \cdots + n_q = w$, i.e., the critical nodes define a partition of the $w$ leaves.
        \end{itemize}
        Furthermore, the first two observations imply that $n_3 + n_4 + \cdots + n_q \le n_2/2 + n_2 / 4 + \cdots < n_2$. Together with the third observation, this implies $\max(n_1, n_2) \ge w/3$, i.e., at least one critical node (either $v_1$ or $v_2$) covers a constant fraction of the $w$ leaves.
        
        \paragraph{Lower bound the average loss of $\ell_{\hat i}$.} Now we give a lower bound on the expected average loss of the chosen model $\ell_{\hat i}$ over the prediction window $[t + 1, t + w]$. By our construction of the sequence, conditioning on the values of the critical nodes in $\T_{\hat i}$, the first $t$ entries of $a^{(\hat i)}$ are independent of entries $a^{(\hat i)}_{t+1}$ through $a^{(\hat i)}_{t+w}$. Furthermore, for each $t' \in [t + 1, t + w]$, conditioning on the value of the critical node $v_i$ that is an ancestor of leaf $t'$, the expectation of $\ell_{\hat i}(z_{t'})$ is exactly the value of $v_i$. Finally, recalling that every node in level $d_i$ is assigned with value $\frac{1}{2} \pm \frac{d_i}{4k}$, we conclude that the expected average loss of $\ell_{\hat i}$ over the entire prediction window is at least
        \[
            \sum_{i=1}^{q}\frac{n_i}{w}\cdot\left(\frac{1}{2} - \frac{d_i}{4k}\right).
        \]
        
        \paragraph{Good events for the remaining $m-1$ models.} On the other hand, we will show that with probability $\Omega(1)$, at least one of the other $m - 1$ models achieves a considerably lower average loss. Let $j^*$ be the index $j \in [q]$ that maximizes $n_j$. Recall that we proved $n_{j^*} \ge w/3$. Let $\Delta \ge 1$ be a parameter to be determined later. For each $i \in [m] \setminus \{\hat i\}$, define $\E_i$ as the event that the following conditions hold simultaneously:
        \begin{itemize}
            \item Every critical node $v_j$ is assigned value $\frac{1}{2} - \frac{d_j}{4k}$ in $\T_i$.
            \item For every node $v$ that has level at most $d_{j^*} + \Delta$ and is a descendant of $v_{j^*}$, $v$ is assigned a value less than $\frac{1}{2}$ in $\T_i$.
            \item Furthermore, if $d_{j^*} + \Delta > k$, the $t'$-th bit of sequence $a^{(i)}$ is zero as long as the leaf with number $t'$ is a descendant of $v_{j^*}$.
        \end{itemize}
        
        Let $\E$ be the union of the events $\{\E_i: i \in [m] \setminus \{\hat i\}\}$. We will show that: (1) conditioning on $\E$, the excess risk is at least $\Omega(\min\{\frac{\Delta}{\log n}, 1\})$; (2) for some $\Delta \ge \Omega(\log\log|\L|)$, event $\E$ happens with probability $\Omega(1)$. These two claims together immediately imply the theorem.
        
        \paragraph{Excess risk is high conditioning on $\E$.} We first consider the case that $d_{j^*} + \Delta > k$. Whenever $\E_i$ happens for some $i \ne \hat i$, the first condition of $\E_i$ guarantees that in $\T_i$, the leaves of each subtree rooted at $v_j$ contributes at most $\frac{n_j}{w}\cdot\left(\frac{1}{2} - \frac{d_j}{4k}\right)$ to the average loss of $\ell_i$ in expectation. Furthermore, the third condition of $\E_i$ implies that the subtree rooted at $v_{j^*}$ contributes exactly zero to the average loss. Thus, the expected excess risk conditioning on $\E_i$ is at least 
        \[
            \sum_{j=1}^{q}\frac{n_j}{w}\cdot\left(\frac{1}{2} - \frac{d_j}{4k}\right)
        -   \sum_{j\in[q]\setminus\{j^*\}}\frac{n_j}{w}\cdot\left(\frac{1}{2} - \frac{d_j}{4k}\right)
        =    \frac{n_{j^*}}{w}\cdot\left(\frac{1}{2} - \frac{d_{j^*}}{4k}\right)
        \ge \frac{1}{3} \cdot \frac{1}{4} = \Omega(1).
        \]
        Similarly, when $d_{j^*} + \Delta \le k$, every subtree rooted at $v_j$ for $j \ne j^*$ still contributes $\frac{n_j}{w}\cdot\left(\frac{1}{2} - \frac{d_j}{4k}\right)$ to the expected average loss. On the other hand, the second condition implies that every descendant of $v_{j^*}$ at level $d_{j^*} + \Delta$ is assigned value $\frac{1}{2} - \frac{d_{j^*} + \Delta}{4k}$. Thus, the expected excess risk is lower bounded by
        \[
            \frac{n_{j^*}}{w}\cdot\left(\frac{1}{2} - \frac{d_{j^*}}{4k}\right) - \frac{n_{j^*}}{w}\cdot\left(\frac{1}{2} - \frac{d_{j^*} + \Delta}{4k}\right)
        =   \frac{n_{j^*}}{w} \cdot \frac{\Delta}{4k}
        \ge \frac{1}{3} \cdot \frac{\Delta}{4k}
        =   \Omega\left(\frac{\Delta}{\log n}\right).
        \]
        Combining these two cases proves the $\Omega(\min\{\frac{\Delta}{\log n}, 1\})$ bound that we claimed.
    
        \paragraph{Event $\E$ happens with a good probability.} Now we turn to lower bounding $\pr{\E_i}$ for each $i$. By our construction of the node values, every critical node gets a value below $\frac{1}{2}$ with probability at least $\frac{1}{2k}$, regardless of the value of its parent. Recalling that $q \le k$, the first condition of $\E_i$ happens with probability at least $(2k)^{-k}$. Furthermore, conditioning on that the first condition holds, the second condition holds if there are no flips within the top $\Delta$ levels in the subtree rooted at $v_{j^*}$. Recall that a flip happens with probability at most $\frac{1}{2}$, and note that there are at most $2\cdot(2^\Delta-1)$ such edges. Thus, the second condition holds with probability at least $2^{-2\cdot(2^\Delta-1)}$. Finally, conditioning on that the first two conditions hold simultaneously, the last condition holds if each of the $2^{k - d_{j^*}} \le 2^\Delta$ samples from the Bernoulli distribution with parameter $1/4$ turns out to be zero. This happens with probability at least $(3/4)^{2^\Delta}$. Therefore, we have
        \[
            \pr{\E_i}
        \ge (2k)^{-k} \cdot 2^{-2\cdot(2^\Delta-1)} \cdot (3/4)^{2^\Delta}.
        \]
        If we choose $\Delta$ such that $\pr{\E_i} \ge \frac{1}{m-1}$, the probability of $\E$ would be lower bounded by
        \[
            1 - \pr{\bigcap_{i\in[m]\setminus\{\hat i\}}\overline{\E_i}}
        \ge 1 - \left(1 - \frac{1}{m-1}\right)^{m-1}
        \ge 1 - 1/e = \Omega(1)
        \]
        as desired. Note that 
        \begin{align*}
            \pr{\E_i} \ge \frac{1}{m-1}
        &\impliedby
            (2k)^{k} \cdot 2^{2(2^\Delta-1)} \cdot (4/3)^{2^\Delta} \le m - 1\\
        &\impliedby
            3\cdot 2^\Delta \le \log_2\left[\frac{m - 1}{(2k)^k}\right].
        \end{align*}
        Moreover, since we assumed that $m \ge n^{\omega(\log\log n)}$ while $(2k)^k = n^{O(\log\log n)}$, for sufficiently large $n$ we have $\frac{m-1}{(2k)^k} \ge \sqrt{m}$. Then,
        \[
            \pr{\E_i} \ge \frac{1}{m-1}
        \impliedby
            3\cdot 2^\Delta \le \log_2\sqrt{m}
        \impliedby
            \Delta \le \log_2\left(\frac{1}{6}\log_2m\right).
        \]
        In other words, we can take $\Delta = \Omega(\log\log|\L|)$ for $\pr{\E_i} \ge \frac{1}{m-1}$ to be satisfied. This proves the second claim and thus the theorem.
    \end{proof}

\section{Proof of the Bounded-Recall Lower Bound}\label{sec:missing-lower-bounded-recall}
\begin{proof}[Proof of Theorem~\ref{thm:bounded-recall-lower}]
    Without loss of generality, we assume that both $n$ and $m$ are powers of two, $m \le n$, and let $k = \log_2n$. Fix a bounded-recall algorithm for selective learning, and let $w$ denote the length of the prediction window chosen by the algorithm. Since every bounded-recall algorithm is also non-adaptive, the window length $w$ is independent of the actual data sequence. Let $\Ical_1 = [2^1, 2^2), \Ical_2 = [2^2, 2^3), \ldots, \Ical_{k-1} = [2^{k-1}, 2^k)$ be a partition of $[2, n)$, and let $p_i$ be the probability that $w \in \Ical_i$.
    
    \paragraph{Existence of large $p_i$'s.} We assume without loss of generality that $p_1 + p_2 + \cdots + p_{k-1} \ge 1/2$; otherwise, we have either $\pr{w = 1} \ge 1/4$ or $\pr{w = n} \ge 1/4$, and we will address these corner cases at the end of the proof. Under this assumption, it is clear that there exist $\log_2m$ different indices $k_1, k_2, \ldots, k_{\log_2m} \in [k - 1]$ such that $p_{k_1} + p_{k_2} + \cdots + p_{k_{\log_2m}} \ge \frac{\log_2m}{2k}$, i.e., with a good probability, the window length $w$ falls into the union of $\log_2m$ intervals $\bigcup_{i=1}^{\log_2m}\Ical_{k_i}$.
    
    \paragraph{Block-based construction of strings.} By Remark~\ref{rem:construct}, we can construct the selective learning instance by generating $m = |\L|$ number sequences $a^{(1)}, \ldots, a^{(m)} \in \{0, 1\}^n$. Then, there always exists an instance where $\ell_i(z_j)$ is exactly $a^{(i)}_j$, i.e., the binary string $a^{(i)}$ denotes the losses of $\ell_i$ on the $n$ data points. We will consider the following $k + 1$ distributions over $\{0, 1\}^n$: $\D_0, \D_1, \D_2, \ldots, \D_k$. Each distribution $\D_i$ is defined by the following procedure: generate $n/2^i$ independent random bits and form a length-$n$ string by repeating each bit $2^i$ times. In other words, each sample from $\D_i$ consists of blocks of size $2^i$, and each block consists of $2^i$ copies of the same random bit. In particular, $\D_0$ is the uniform distribution over $\{0, 1\}^n$, while $\D_k$ is uniform over $\{00\cdots00, 11\cdots11\}$.
    
    \paragraph{The hard instance.} The hard instance against the fixed bounded-recall learner will be constructed by taking a few independent strings from some of the $\D_i$'s and then randomly permuting the sequences. Specifically, the $m$ strings consists of the following:
    \begin{itemize}
        \item $m/2$ independent strings drawn from $\D_{k_1 - 1}$.
        \item $m/4$ independent strings drawn from $\D_{k_2 - 1}$.
        \item $\cdots$
        \item One string drawn from $\D_{k_{\log_2m} - 1}$.
        \item One string that consists of $n$ zeros.
    \end{itemize}
    Indeed, there are $m/2 + m/4 + \cdots + 1 + 1 = m$ strings in total. Finally, these $m$ strings are randomly permuted such that the learner cannot tell in advance the distribution from which each $a^{(i)}$ is drawn. In the following, we will show that whenever the chosen window length $w$ falls into some $\Ical_{k_i}$ ($i \in [\log_2 m]$), the expected excess risk is $\Omega(1)$. This immediately proves that the overall expected excess risk is at least $(p_{k_1} + p_{k_2} + \cdots + p_{k_{\log_2 m}}) \cdot \Omega(1) = \Omega\left(\frac{\log|\L|}{\log n}\right)$.

    \paragraph{Excess risk is high whenever a bad string is chosen.} Suppose that the window length $w$ falls into $\Ical_{k_i} = [2^{k_i}, 2^{k_i+1})$. In this case, we say that each of the $m$ strings is ``bad'', if it is drawn from either of $\D_{k_1 - 1}, \D_{k_2 - 1}, \ldots, \D_{k_i - 1}$; otherwise the string is ``good''. Note that there are exactly $m/2^i$ good strings.
    
    We claim that, conditioning on that the algorithm chooses a bad string, it incurs an expected excess risk of $\Omega(1)$. To see this, note that if we choose a bad string from $\D_{k'}$ for some $k' \le k_i - 1$, the string consists of blocks of size $2^{k'} \le 2^{k_i - 1} \le w / 2$. Thus, at least half of the bits in the prediction window have not been observed, and are thus uniformly distributed over $\{0, 1\}$ even conditioning on the previous observations. It follows that the expected loss of the chosen model over the prediction window is at least $\frac{1}{2} \cdot \frac{1}{2} = \frac{1}{4}$. On the other hand, since one of the $m$ strings is the all-zero string, the best achievable loss over the window is $0$. This shows that the expected excess risk conditioning on choosing a bad string is at least $1/4 = \Omega(1)$.
    
    In the rest of the proof, we will show that the probability of choosing a bad string is also $\Omega(1)$. The intuition behind this is that there are exactly $m/2^i$ good strings, yet there are also $m/2^i$ bad strings that are drawn from $\D_{k_i - 1}$. Since the bounded-recall learner only sees the most recent $w < 2^{k_i + 1}$ entries of the sequences, there is a decent chance that each such bad string ``looks like'' a good string, and thus it is impossible to distinguish these bad strings from the good ones perfectly.
    
    \paragraph{Classify the good strings.} Let $M = m/2^i$ denote both the number of good strings, which is also the number of strings drawn from $\D_{k_i - 1}$. Suppose that the bounded-recall algorithm chooses the model based on the most recent $w$ entries $t - w + 1, t - w + 2, \ldots, t$ of the sequences. We will show that, only given these $w < 2^{k_i+1}$ entries, it is impossible to identify a good string with probability $1 - o(1)$. We start by classifying the good strings into at most four types:
    \begin{itemize}
        \item Type 0: The single all-zero string, in which the observed entries are always $w$ zeros.
        \item Type 1: The strings drawn from some $\D_{k'}$ such that the interval $[t-w+1,t]$ intersects only one block of size $2^{k'}$, i.e., $\lfloor(t-w)/2^{k'}\rfloor = \lfloor(t-1)/2^{k'}\rfloor$. In this case, the $w$ observed entries of the string are either all zeros or all ones, with equal probability.
        \item Type 2: The strings drawn from some $\D_{k'}$ such that the interval $[t-w+1,t]$ intersects exactly two blocks of size $2^{k'}$. This happens if and only if $\lfloor(t-w)/2^{k'}\rfloor$ and $\lfloor(t-1)/2^{k'}\rfloor$ differ by one. In this case, the $w$ observed bits contain two parts (of possibly different lengths), and each part consists of the same random bit.
        \item Type 3: The strings drawn from some $\D_{k'}$ such that the interval $[t-w+1,t]$ intersects exactly three blocks of size $2^{k'}$. This can only happen if $k' = k_i$.
    \end{itemize}
    For instance, suppose that $k_i = 2$, $w = 6 \in [4, 8)$ and $t = 9$, so that the most recent $6$ entries are those with indices $4$ through $9$. Then, each good string $a \sim \D_4$ consists of blocks of length $16$, and the interval $[4, 9]$ only intersects the first block. Thus, each such string is of Type 1. Each good string $a \sim \D_3$ is of Type 2, and the observed window $(a_4, \ldots, a_9)$ is identically distributed as $(b_1, b_1, b_1, b_1, b_1, b_2)$ where $b_1$ and $b_2$ are independent random bits. Furthermore, every good string $a \sim \D_2$ is of Type 3, and the observed bits follow the distribution of $(b_1, b_2, b_2, b_2, b_2, b_3)$ for random bits $b_1, b_2, b_3$. 
    
    The crucial observation is that, for each of the $M$ bad strings sampled from $\D_{k_i - 1}$, since $w < 2^{k_i + 1} = 4 \cdot 2^{k_i - 1}$, the observed window of length $w$ intersects at most $4$ blocks of length $2^{k_i - 1}$. Thus, for each $j \in \{0, 1, 2, 3\}$, there exists some event $\E_j$ that happens with probability $2^{-4} = \Omega(1)$ such that conditioning on $\E_j$, the observed $w$ bits in the bad string is identically distributed as those in a good string of Type $j$. More concretely, let $\E_0$ be the event that every size-$2^{k_i-1}$ block that intersects $[t-w+1, t]$ consists of zeros. Then, $\pr{\E_0} \ge 2^{-4}$ and conditioning on $\E_0$, this bad string is indistinguishable from the all-zero good string of Type 0. Similarly, we can define $\E_1$ as the event that all the relevant size-$2^{k_i-1}$ blocks receive the same bit, and the conditional distribution of the observed window is also identical to that of a Type 1 good string.
    
    \paragraph{Lower bound the probability of choosing a bad string.} Therefore, we can partition the $M$ bad strings from $\D_{k_i - 1}$ into four groups numbered $0, 1, 2, 3$, each consisting of $M/4$ strings. For each $j \in \{0, 1, 2, 3\}$, let $m_j$ be the number of good strings of Type $j$, and let random variable $m'_j$ be the number of bad strings from $\D_{k_i - 1}$ for which the event $\E_j$ happens. Our previous observation implies that for every $j$, $\Ex{m'_j} \ge (M/4) \cdot 2^{-4} = M/64$. By Markov's inequality, $\pr{m'_j \ge M/128} \ge 1/32$. Since the four groups of strings are independent, it holds with probability $1/32^4 = \Omega(1)$ that $m'_0, \ldots, m'_3$ are all at least $M/128$. Let $p_j$ be the probability that the algorithm chooses a particular Type $j$ string. This is well-defined because the strings are randomly permuted and all Type $j$ strings are indistinguishable. It also holds that the algorithm would choose each of the $m'_j$ bad strings with probability exactly $p_j$. Therefore, we have
    \[
        \sum_{j=0}^{3}p_j(m_j + m'_j) \le 1.
    \]
    Then, the probability of choosing a good string is upper bounded by
    \[
        \sum_{j=0}^{3}p_jm_j
    =   \sum_{j=0}^{3}p_j(m_j + m'_j)\cdot\frac{m_j}{m_j + m'_j}
    \le \frac{1}{1+1/128}\sum_{j=0}^{3}p_j(m_j + m'_j)
    \le \frac{128}{129} = 1 - \Omega(1),
    \]
    where the second step follows from $m_j \le M$ and $m'_j \ge M / 128$.
    This proves that a bad string is chosen with probability at least $\Omega(1)$, which in turn implies the desired lower bound.
    
    \paragraph{Handle the $w = 1$ and $w = n$ cases.} Finally, we take care of the corner cases where either $w = 1$ or $w = n$ holds with probability $\Omega(1)$. In the former case, consider an instance formed by $m$ independent strings from $\D_0$. Then, regardless of the previous observations, the loss of each model on the next data point is a random bit, so the expected loss of the chosen model is $1/2$. On the other hand, as long as $m \ge 2$, the expected minimum loss is $2^{-m} \le 1/4$ and thus the excess risk is $\Omega(1)$ in expectation.
    
    Similarly, for the latter case that $w = n$, we could consider $m$ strings drawn from $\D_k$. Then, the average loss of the chosen model over the whole horizon is either $0$ or $1$ with equal probability, while the minimum possible loss is $2^{-m}$ in expectation. So the expected excess risk is also $\Omega(1)$.
\end{proof}

\input{mean}

\section{Other Omitted Proofs}
\subsection{Missing Proof from Section~\ref{sec:bounded-recall}}\label{sec:missing-bounded-recall}
    In the following, we prove that the log-sum-exp function defined as $f(x) = \ln\left(\sum_{i=1}^{d}e^{-\alpha x_i}\right)$ for $\alpha > 0$ is $(4\alpha)$-self-concordant over $S = [0, 1]^d$, thus completing the proof of Theorem~\ref{thm:bounded-recall-upper}. The proof is similar to the one in Appendix A.6 of~\cite{tran2015composite}; we provide the full proof below for completeness. The only slight difference is that the previous proof works with the 2-norm rather than the infinity norm.
    
    \begin{proof}
        Fix $x, y \in [0, 1]^d$ and define $g(t) \coloneqq f(x + t(y-x))$. Fix $t \in [0, 1]$ and let $\Delta = y - x$, $a_i = e^{-\alpha(x_i + t\Delta_i)}$. It then follows from a direct calculation that
        \[
            g'(t) = -\alpha\cdot\frac{\sum_{i=1}^{d}a_i\Delta_i}{\sum_{i=1}^{d}a_i},
        \]
        \[
            g''(t) = \alpha^2\cdot\frac{\sum_{1 \le i < j \le d}a_ia_j(\Delta_i - \Delta_j)^2}{(\sum_{i=1}^{d}a_i)^2},
        \]
        and
        \[
            g'''(t) = -\alpha^3\cdot\frac{\sum_{1 \le i < j \le d}a_ia_j(\Delta_i - \Delta_j)^2\left[\sum_{k=1}^{d}a_k(\Delta_i + \Delta_j - 2\Delta_k)\right]}{(\sum_{i=1}^{d}a_i)^3}.
        \]
        Note that since $x, y \in [0, 1]^d$ and $\Delta = y - x$, $\Delta_i \in [-1, 1]$ for each $i \in [d]$. It then follows that $|\Delta_i + \Delta_j - 2\Delta_k| \le 4$ for any $i, j, k \in [d]$. Thus, we can bound $|g'''(t)|$ as follows:
        \begin{align*}
            |g'''(t)|
        &\le \alpha^3\cdot\frac{\sum_{1 \le i < j \le d}a_ia_j(\Delta_i - \Delta_j)^2\left[\sum_{k=1}^{d}a_k|\Delta_i + \Delta_j - 2\Delta_k|\right]}{(\sum_{i=1}^{d}a_i)^3}\\
        &\le \alpha^3\cdot\frac{\sum_{1 \le i < j \le d}a_ia_j(\Delta_i - \Delta_j)^2\left(\sum_{k=1}^{d}4a_k\right)}{(\sum_{i=1}^{d}a_i)^3}
        =   4\alpha g''(t).
        \end{align*}
        Therefore, the log-sum-exp function is $4\alpha$-self-concordant over $[0, 1]^d$.
    \end{proof}

\subsection{Missing Proofs from Section~\ref{sec:discuss}}\label{sec:missing-discuss}
We start by restating and proving Proposition~\ref{prop:vc-not-upper}, which states that there exist model classes of a small VC dimension that cannot be learned in the selective learning setting. While the following proof constructs an infinite instance space $\X$ and an infinite model class $\F$, it is easy to discretize the interval $[0, 1]$ and work with finite sets $\X$ and $\F$ of size $O(2^n)$.

\vspace{6pt}

\noindent\textbf{Proposition~\ref{prop:vc-not-upper}~}{\it
    There exists a model class $\F$ of VC dimension $1$ such that for any sequence length $n$, no learning algorithm could achieve an excess risk smaller than $\frac{1}{2}$ in expectation for class $\F$.
}

\begin{proof}
    Consider the instance space $\X = [0, 1]$ and the family of all threshold functions over $\X$: $\F = \{f_{\theta}: \theta \in [0, 1]\}$, where $f_{\theta}:[0,1]\to\{0, 1\}$ is defined as $f_{\theta}(x) = \I{x \ge \theta}$. It is easy to verify that $\F$ has a VC dimension of $1$.

    Now we prove that no learner can achieve a sub-constant excess risk on $\F$. Fix an integer $n$ and consider the following procedure for generating a random instance of selective learning over $\F$. Let $I_1 = [0, 1]$. For each $i = 1, 2, \ldots, n$, the $i$-th data point $x_i$ is chosen as the middle point of the interval $I_i$, and the label $y_i$ is sampled uniformly at random from $\{0, 1\}$. Then, the next interval $I_{i+1}$ is defined to make the label $y_i$ consistent, i.e., for $I_i = [l_i, r_i]$, we set $I_{i+1} = [l_i, (l_i + r_i) / 2]$ if $y_i = 0$ and $I_{i+1} = [(l_i + r_i) / 2, r_i]$ if $y_i = 1$.
    
    By construction, there exists a model in $\F$ that is consistent with all the labeled data points $\{(x_i, y_i)\}_{i=1}^{n}$; indeed, one can verify that $f_{\theta^*} \in \F$ for $\theta^* = (l_{n+1} + r_{n+1}) / 2$ would be such a model. On the other hand, for any learning algorithm (that can be potentially randomized), whenever it makes a prediction at some time step $t$, the remaining labels $y_{t+1}, y_{t+2}, \ldots, y_n$ are, by our construction, uniformly random bits independent of all the observations up to time $t$. This implies that no matter how the learner selects the prediction window and the model, the expected loss of the learner is $\frac{1}{2}$. By an averaging argument, there exists an instance of selective learning over $\F$ on which the learner incurs an expected excess risk of at least $\frac{1}{2}$.
\end{proof}

In the following we restate and prove Proposition~\ref{prop:realizable}.

\vspace{6pt}

\noindent\textbf{Proposition~\ref{prop:realizable}~}{\it
    For the selective learning problem, under the promise that there exists a model $\ell^*\in\L$ such that $\ell^{*}(z_i) = 0$ for every $i \in [n]$, there is an algorithm with expected excess risk of $O\left(\frac{\log|\L|}{n}\right)$.
}

\begin{proof}
    We consider the following simple learning procedure:
    \begin{itemize}
        \item Sample $t \in \{0, 1, \ldots, n - 1\}$ uniformly at random and observe $z_1, z_2, \ldots, z_t$.
        \item Among all models in $\L$ with a zero loss on each of the observed data, choose a model uniformly at random.
        \item Predict that the chosen model works well on the next data point $z_{t+1}$, i.e., the prediction window is of length $1$.
    \end{itemize}
    
    For each $i \in \{0, 1, \ldots, n\}$, let $m_i$ denote the number of models in $\L$ that are consistent with (i.e., have a zero loss on) the first $i$ data points. By definition, $m_0 = |\L|$ and $m_{n} \ge 1$. Note that conditioned on that the algorithm picks stopping time $t$, the probability of incurring a non-zero loss on $z_{t+1}$ is exactly $1 - m_{t+1}/m_t$. Since the loss is at most $1$, the expected excess risk is upper bounded as follows:
    \[
        1 - \frac{1}{n}\sum_{t=0}^{n-1}\frac{m_{t+1}}{m_t}
    \le 1 - \left(\prod_{t=0}^{n-1}\frac{m_{t+1}}{m_t}\right)^{1/n}
    \le 1 - e^{-\ln|\L|/n}
    \le \frac{\ln|\L|}{n}.
    \]
    The first step applies the AM-GM inequality. The second step follows from $m_0 = |\L|$ and $m_n \ge 1$. The last step applies $e^{-x} \ge 1 - x$. This proves the $O\left(\frac{\log|\L|}{n}\right)$ upper bound.
\end{proof}

%% file: mean.tex
\section{Generalized Mean Prediction}\label{sec:mean}
    In this section, we analyze Algorithm~\ref{alg:mean-prediction} and prove Theorem~\ref{thm:mean-prediction}. The analysis builds on the following lemma which bounds the $f$-loss in terms of the function value of $f$.
    
    \begin{lemma}\label{lem:loss-bound}
        For any $C_1$-self-concordant convex function $f: S \to \R$, $u, v \in S$ and $\mu = (u + v) / 2$, it holds that $D_f(u, v) + D_f(v, u) \le (2C_1 + 4)[f(u) + f(v) - 2f(\mu)]$.
    \end{lemma}
    
    \begin{proof}
        Fix $\mu = (u + v) / 2$ and rewrite $u = \mu + \Delta$ and $v = \mu - \Delta$. Define
        \begin{align*}
            g(\Delta) &\coloneqq
            (2C_1 + 4)[f(u) + f(v) - 2f(\mu)] - [D_f(u, v) + D_f(v, u)]\\
        &=   (2C_1 + 4)[f(\mu + \Delta) + f(\mu - \Delta) - 2f(\mu)] - 2[\nabla f(\mu + \Delta) - \nabla f(\mu - \Delta)]^{\top}\Delta.
        \end{align*}
        We will prove that $\nabla g(\Delta)^{\top}\Delta \ge 0$ for any $\Delta$ such that $\mu + \Delta, \mu - \Delta \in S$. Note that this inequality, together with the mean value theorem, implies that for any valid $\Delta$,
        \[
            g(\Delta)
        =   g(0) + \nabla g(\alpha\Delta)^{\top}\Delta
        =   g(0) + \frac{1}{\alpha}\nabla g(\alpha\Delta)^{\top}(\alpha\Delta)
        \ge g(0) = 0
        \]
        holds for some $\alpha \in (0, 1)$, which proves the second part of the lemma.
        
        To show that $\nabla g(\Delta)^{\top}\Delta \ge 0$, we expand the left-hand side into
        \begin{align*}
            &~(2C_1 + 2)[\nabla f(\mu + \Delta) - \nabla f(\mu - \Delta)]^{\top}\Delta - 2\Delta^{\top}[\nabla^2 f(\mu + \Delta) + \nabla^2 f(\mu - \Delta)]\Delta\\
        =   &~(C_1 + 1)[\nabla f(u) - \nabla f(v)]^{\top}(u-v) - \frac{1}{2}(u-v)^{\top}[\nabla^2 f(u) + \nabla^2 f(v)](u-v).
        \end{align*}
        Define $h:[0, 1]\to \R$ as $h(t) = f(v + t(u-v))$. Then, $\nabla g(\Delta)^{\top}\Delta$ can be further written as
        \[
            (C_1 + 1)[h'(1) - h'(0)] - \frac{1}{2}[h''(0) + h''(1)].
        \]
        Since $f$ is $C_1$-self-concordant, it holds by definition that $|h'''(t)| \le C_1 h''(t)$. Moreover, since $f$ is convex, $h''(t)$ is non-negative. By solving the differential inequality, we have $h''(t) \ge e^{-C_1t}h''(0)$. It then follows that
        \[
            h'(1) - h'(0)
        =   \int_0^1h''(t)~\mathrm{d} t
        \ge h''(0)\int_0^1e^{-C_1t}~\mathrm{d} t
        =   \frac{1-e^{-C_1}}{C_1}h''(0)
        \ge \frac{h''(0)}{C_1 + 1}.
        \]
        The last step applies the inequality $\frac{1-e^{-x}}{x} > \frac{1}{1+x}$ for $x > 0$.
        Similarly, we have $h'(1) - h'(0) \ge \frac{h''(1)}{C_1 + 1}$. Therefore, it holds that
        \begin{align*}
            \nabla g(\Delta)^{\top}\Delta
        &=   (C_1 + 1)[h'(1) - h'(0)] - \frac{1}{2}[h''(0) + h''(1)]\\
        &\ge (C_1 + 1)\cdot\frac{1}{2}\left[\frac{h''(0)}{C_1 + 1} + \frac{h''(1)}{C_1 + 1}\right] - \frac{1}{2}[h''(0) + h''(1)]
        =   0,
        \end{align*}
        which completes the proof.
    \end{proof}

    The proof of Theorem~\ref{thm:mean-prediction} is based on an induction similar to the analyses of selective prediction in~\cite{drucker2013high,qiao2019theory}.
    \begin{proof}[Proof of Theorem~\ref{thm:mean-prediction}]
        We prove by induction that on any sequence with length $2^k$ and average $\mu \in S$, the expected $f$-loss of Algorithm~\ref{alg:mean-prediction} is at most $(4C_1 + 8)\cdot\frac{f_{\max} - f(\mu)}{k}$, where $f_{\max} = \sup_{u \in S}f(u)$. As a direct corollary, since $f$ is $C_0$-bounded, the expected $f$-loss of Algorithm~\ref{alg:mean-prediction} on any sequence of length $n$ is at most
        $(4C_1 + 8)\cdot\frac{f_{\max} - f(\mu)}{k}
        \le \frac{C_0(4C_1 + 8)}{\lfloor\log_2 n\rfloor} = O\left(\frac{C_0(C_1 + 1)}{\log n}\right)$.
        
        Indeed, in the base case $k = 1$, Algorithm~\ref{alg:mean-prediction} reduces to predicting that $x_2$ equals $x_1$. By Lemma~\ref{lem:loss-bound}, the $f$-loss is given by
        \[
            D_f(x_1, x_2) + D_f(x_2, x_1)
        \le (2C_1 + 4)[f(x_1) + f(x_2) - 2f(\mu)]
        \le (4C_1 + 8)[f_{\max} - f(\mu)].
        \]
        For the inductive step, let $u$ and $v$ denote the averages of the two halves of the sequence. With probability $1/k$, the algorithm predicts that the second half of the sequence has an average of $u$, and thus, by Lemma~\ref{lem:loss-bound}, the conditional $f$-loss is
            $D_f(u, v) + D_f(v, u) \le (2C_1 + 4) [f(u) + f(v) - 2f(\mu)]$.
        With the remaining probability of $1-1/k$, the algorithm can be equivalently viewed as randomly picking one of the two halves of the sequence, and running the same algorithm for sequences of length $2^{k-1}$. By the inductive hypothesis, the conditional expected value of the $f$-loss is at most
        \[
            \frac{1}{2}\left[(4C_1 + 8)\cdot\frac{f_{\max} - f(u)}{k-1}
        +   (4C_1 + 8)\cdot\frac{f_{\max} - f(v)}{k-1}\right]
        =   (2C_1 + 4)\cdot\frac{2f_{\max} - f(u) - f(v)}{k-1}.
        \]
        Therefore, the expected $f$-loss of the algorithm on the original sequence of length $2^k$ is at most
        \[
            \frac{1}{k}(2C_1 + 4)[f(u) + f(v) - 2f(\mu)]
        +   \frac{k-1}{k}(2C_1 + 4)\cdot\frac{2f_{\max} - f(u) - f(v)}{k-1}
        =   (4C_1 + 8)\cdot\frac{f_{\max} - f(\mu)}{k},
        \]
        which completes the induction.
    \end{proof}